\documentclass[runningheads]{llncs}

\newcommand {\rf} {\mathit{rank}}

\newcommand {\lingconc} {\mathcal{S}}

\newcommand {\nott} {\lnot}

\newcommand {\sx} {\langle}
\newcommand {\dx} {\rangle}

\newcommand {\emme} {\mathcal{M}}
\newcommand {\enne} {\mathcal{N}}

\newcommand {\elle} {\mathcal{L}}

\newcommand {\unione} {\cup}

\newcommand {\tc} {\mid}

\newcommand{\tip}{{\bf T}}

\newcommand{\alc}{\mathcal{ALC}}
\newcommand{\shiq}{\mathcal{SHIQ}}
\newcommand{\alct}{\mathcal{ALC}+\tip}

\newcommand{\el}{\mathcal{EL}^{\bot}}

\newcommand{\eltm}{\mathcal{EL}^{\bot} \tip_{min}}

\newcommand{\alctr}{\mathcal{ALC}+\tip_{\textsf{\tiny R}}}
\newcommand{\alctre}{{\mathcal{ALC}^{\Ra}\tip}_{\tiny E} }
\newcommand{\alctrse}{{\mathcal{ALC}^{\Ra}\tip}_{\tiny S} }

\newcommand{\dlltm}{\mbox{\em DL-Lite}_{\mathit{c}}\tip_{min}}

\newcommand{\be}{\begin{enumerate}}
\newcommand{\ee}{\end{enumerate}}

\newcommand{\hide}[1]{}

\def \cases{\left \{\begin{array}{l}}
\def \endcases{\end{array}\right .}

\newcommand {\Ra} {{\bf R}}

\newcommand {\bes} {\begin{description}}
\newcommand{\ens} {\end{description}}

\newcommand {\beq} {\begin{quote}}
\newcommand {\enq} {\end{quote}}
\newcommand {\bit} {\begin{itemize}}
\newcommand {\enit} {\end{itemize}}

\newcommand {\bbox}{\square}

\newenvironment{pozz}{\color{black}}{\color{black}}
\usepackage{times}
\usepackage{helvet}
\usepackage{courier}
\usepackage{times}
\usepackage{graphicx}
\usepackage{latexsym}
\usepackage{graphicx}
\usepackage{color}
\usepackage{amssymb}
\usepackage{colortbl}
\usepackage{amsmath}
\usepackage{microtype}

\title{Reasoning about multiple aspects in DLs: \\
Semantics and Closure Construction}

\author{Laura Giordano \inst{1} \and Valentina Gliozzi \inst{2}}

\authorrunning{L. Giordano and V.Gliozzi}

\titlerunning{ }

\institute{DISIT - Universit\`a del Piemonte Orientale - 
 Alessandria, Italy - \email{laura.giordano@uniupo.it} \and
Center for Logic, Language and Cognition and Dipartimento di Informatica \\ 
 Universit\`a di Torino, Italy\\ \email{valentina.gliozzi@unito.it}
}

\begin{document}
\maketitle

 \begin{abstract} 
Starting from the observation that rational closure has the undesirable property of being an ``all or nothing" mechanism,
we here propose a multipreferential semantics, which enriches the preferential semantics underlying rational closure 
in order to separately deal with the inheritance of different properties in an ontology with exceptions.
We provide a multipreference closure mechanism which is sound with respect to the multipreference semantics.
\end{abstract}

\section{Introduction}
Reasoning with exceptions has been widely studied within non-monotonic reasoning in Artificial Intelligence and in Description Logics (DLs).
In particular, a lot of work has been devoted to extending DLs with non-monotonig formalisms to allow reasoning about {prototypical properties} of individuals
\cite{Straccia93,baader95b,donini2002,eiter2004,lpar2007,AIJ13,AIJ15,kesattler,sudafricaniKR,FI09,bonattilutz,casinistraccia2010,rosatiacm,hitzlerdl,KnorrECAI12,CasiniDL2013,bonattiAIJ15}.  

In this paper we propose an extension of {\em rational closure} \cite{whatdoes} for dealing with multiple preferences.
One of the main difficulties of rational closure to deal with inheritance of  defeasible properties of concepts is the fact that one cannot
reason property by property:  if a subclass of a class $C$ is exceptional with respect to $C$ for a given property, 
it does not inherit any of the defeasible properties of $C$.  
This is the ``all or nothing" behavior of rational closure. 

Consider the following version of the classic birds/penguins example: 

{ \em

Typical birds fly

Penguins are birds

Typical penguins do not fly

Typical birds have nice feather}

\noindent By rational closure, penguins (being exceptional birds concerning the property of  flying) do not inherit any of the typical properties of birds.

On the contrary, one could expect penguins to inherit the property of having nice feather, for which they are not exceptional.
More generally, we would like to reason independently on the inheritance of the properties of one concept by a more specific one.
This is what Lehmann calls  ``presumption of independence" \cite{Lehmann95}: even if typicality is lost with respect to one property,
we may still presume typicality with respect to another, unless there is reason to the contrary.


In this paper, we address this problem from the semantic point of view.
Starting from the preferential semantics underlying rational closure introduced by Lehmann and Magidor \cite{whatdoes} for propositional logics and 
extended to the description logic $\alc$ in \cite{AIJ15}, we consider a preferential semantics 
which allows to reason about typicality with respect to different aspects.
In this semantics, an individual can be more preferred than another
when considering one aspect (e.g. being a good student), while less preferred when considering another aspect (e.g. being a good citizen).
In this enriched preferential semantics, different preference relations among domain elements are introduced, each one describing the preference of an individual over another one with respect to a given aspect/property.
We show that this semantics is a strengthening of rational closure. 

We provide a syntactic construction which is built over the rational closure and that we call { \em multipreference closure}. 
The multipreference closure is proved to be a sound construction for reasoning with multiple preferences, 
thus providing a sound approximation of the multipreference semantics.
As we will see, this construction is strongly  related with the lexicographic closure, proposed by Lehmann \cite{Lehmann95}
and extended to the description logic $\alc$ by Casini and Straccia \cite{Casinistraccia2012},
but it exploits a different specificity ordering,
as it will become clear in Section \ref{sec:Lexicographic}, where we compare the two constructions.
Another approach related to ours is  the logic of overriding ${\cal DL}^N$ by Bonatti, Faella, Petrova and Sauro \cite{bonattiAIJ15},
a nonmonotonic description logic which also allows reasoning independently about different defeasible properties. 
A difference with ${\cal DL}^N$ is that  ${\cal DL}^N$ leads to the inconsistency of prototypical concepts (thus requiring a repair of the KB) in those cases when a conflict among defaults cannot be solved by overriding. In our approach, as in the lexicographic closure, such conflicts are silently removed considering (skeptically) what holds in all the alternative bases (i.e., maximal consistent sets of defeasible inclusions), while ${\cal DL}^N$ computes a single base.
We discuss the relations between our approach and ${\cal DL}^N$ in Section  \ref{sec:Lexicographic},
where we also suggest that the MP-closure construction could be further approximated by a construction which only requires a polynomial number of entailment checks in $\alc$, following the approach in \cite{GiordanoICTCS2017}. \normalcolor

We will proceed as follows. In Section  \ref{sez:semantica} we recall the rational closure for description logics and its semantics.
In Sections \ref{sec:W-enriched} and \ref{sec:S-enriched}, we define the multipreference semantics by introducing the notions of enriched and strongly enriched models
of a knowledge base.
In Section \ref{sec:closure} we develop the multipreference closure construction, we show that it is sound with respect  to the semantics,
and we compare it with the lexicographic closure, with the logic of overriding ${\cal DL}^N$ \cite{bonattiAIJ15}, and  with related constructions. 
Sections \ref{Related work} and \ref{Conclusions} conclude the paper by assessing the contribution with respect to related work.


\section{The Rational Closure and its Semantics}\label{sez:semantica}
Let us briefly recall 
the logic $\alctr$  which is at the basis of a rational closure construction  proposed in \cite{AIJ15} for $\alc$.
The intuitive idea of $\alctr$ is to extend the standard description logic $\alc$ with concepts of the form $\tip(C)$,  whose intuitive meaning is that
$\tip(C)$ selects the {\em typical} instances of a concept $C$, to distinguish between the properties that
hold for all instances of concept $C$ ($C \sqsubseteq D$), and those that only hold for the typical
 instances ($\tip(C) \sqsubseteq D$).  
 
Let ${N_C}$ be a set of {\em concept names}, ${N_R}$ a set of {\em role names} and ${N_I}$ a set of {\em individual names}. 
The $\alctr$ language is defined as follows:
 $C_R:= A \tc \top \tc \bot \tc  \nott C_R \tc C_R \sqcap C_R \tc C_R \sqcup C_R \tc \forall R.C_R \tc \exists R.C_R$, and
   $C_L:= C_R \tc  \tip(C_R)$, where $A\in N_C$
and $R\in N_R$.  
    A {\em knowledge base} (KB) is a pair $K=( {\cal T}, {\cal A} )$, where the  TBox ${\cal T}$ contains a finite set
of  concept inclusions  $C_L \sqsubseteq C_R$, and the  ABox ${\cal A}$
contains a finite set of assertions of the form {$C_R(a)$ } and $aRb$, where $a, b \in N_I$ are individual names.
In the following we will call {\em non-extended concepts}  the concepts $C_R$ of the language, which do not contain the $\tip$ operator. \normalcolor

\noindent The semantics of $\alctr$  is defined
in terms of 
ranked models similar to those introduced in \cite{whatdoes}:
 ordinary models of $\alc$ are equipped with a \emph{preference relation} $<$ on
the domain, whose intuitive meaning is to compare the ``typicality''
of domain elements: $x < y$ means that $x$ is more typical than
$y$. Typical members of a concept $C$, instances of
$\tip(C)$, are the members $x$ of $C$ that are minimal with respect
to $<$ (such that there is no other member of $C$
more typical than $x$). In rational models $<$ is further assumed to be {\em modular} (i.e., for all $x, y, z \in \Delta$, if
$x < y$ then either $x < z$ or $z < y$) and {\em well-founded} \footnote{Since $\alctr$ has the finite model property, this is equivalent to having the Smoothness Condition, as shown in \cite{AIJ15}. We choose this formulation because it is simpler.} (i.e., there is no infinite $<$-descending chain, so that, if $S \neq \emptyset$, also $min_<(S) \neq \emptyset$). Ranked models characterize
$\alctr$.

\vspace{-0.15cm}
\begin{definition}[Semantics of $\alctr$ \cite{AIJ15}]\label{semalctr} A model $\emme$ of $\alctr$ is any
structure $\langle \Delta, <, I \rangle$ where: $\Delta$ is the
domain;   $<$ is an irreflexive, transitive,  modular and well-founded relation over
$\Delta$. 
$I$ is an interpretation function that maps each
concept name $C \in R_C$ to $C^I \subseteq \Delta$, each role name $R \in N_R$
to  $R^I \subseteq \Delta^I \times \Delta^I$
and each individual name $a \in N_I$ to $a^I \in \Delta$.
For concepts of
$\alc$, $C^I$ is defined in the usual way. For the $\tip$ operator, we have
$(\tip(C))^I = min_<(C^I)$.
\end{definition}
\vspace{-0.2cm}

As shown in \cite{AIJ15}, the logic $\alctr$ enjoys the  finite model property and finite $\alctr$ models can be equivalently defined by postulating 
the existence of
a function $k_{\emme}: \Delta \longmapsto \mathbb{N}$, where $k_{\emme}$ assigns a finite rank to each world: the rank $k_{\emme}$  of a domain element $x \in \Delta$ is the
length of the longest chain $x_0 < \dots < x$ from $x$
to a minimal $x_0$ (s. t. there is no ${x'}$ with  ${x'} < x_0$). The rank $k_\emme(C_R)$ of a concept $C_R$ in $\emme$ is $i = min\{k_\emme(x):
x \in C_R^I\}$.

A model $\emme$ satisfies a knowledge base $K=( {\cal T}, {\cal A} )$ if it satisfies  its TBox (and for all  inclusions $C \sqsubseteq D \in {\cal T}$, it $C^I \subseteq D^I$ holds), 
and its ABox (for all $C(a) \in {\cal A}$,  $a^I \in C^I$ and, for all $aRb \in {\cal A}$,  $(a^I,b^I) \in R^I$).
%
%
%
%
%
 A query $F$ (either an assertion $C_L(a)$ or an inclusion relation $C_L \sqsubseteq C_R$) is logically (rationally) entailed by a knowledge base $K$ ($K \models_{\alctr} F$) if $F$ holds in all models satisfying $K$. 

Although the typicality operator $\tip$ itself  is nonmonotonic (i.e.
$\tip(C) \sqsubseteq D$ does not imply $\tip(C \sqcap E)
\sqsubseteq D$), the logic $\alctr$ is monotonic: what is logically entailed by $K$ is still entailed by any $K'$ with $K \subseteq K'$.

In \cite{dl2013,AIJ15} a non monotonic construction of rational closure has been defined for $\alctr$, extending 
the notion of rational closure proposed in the propositional context by Lehmann and Magidor \cite{whatdoes}.  The definition is based on the notion of {\em exceptionality}. Roughly speaking $\tip(C) \sqsubseteq D$ holds (is included in the rational closure) of $K$ if $C$ (indeed, $C \sqcap D$) is less exceptional than $C \sqcap \neg D$. We briefly recall this construction and we refer to \cite{dl2013,AIJ15} for full details.
Here we only consider rational closure of TBox, defined as follows.

\begin{definition}[Exceptionality of concepts and inclusions]\label{definition_exceptionality}
Let $E$ be a TBox and $C$ a concept. $C$ is
said to be {\em exceptional} for $E$ if and only if $E \models_{\alctr} \tip(\top) \sqsubseteq
\neg C$. A \tip-inclusion $\tip(C) \sqsubseteq D$ is exceptional for $E$ if $C$ is exceptional for $E$. The set of \tip-inclusions of $E$ which are exceptional for $E$ will be denoted
as $\mathcal{E}$$(E)$.
\end{definition}

\noindent Given a $\alctr$  TBox,
it is possible to define a sequence of non increasing subsets of
a TBox ${\cal T}$ ordered according to the exceptionality of the elements $E_0 \supseteq E_1, E_1 \supseteq E_2, \dots$ by letting $E_0 ={\cal T}$ and, for
$i>0$, $E_i=\mathcal{E}$$(E_{i-1}) \unione \{ C \sqsubseteq D \in \mbox{TBox}$ s.t. $\tip$ does not occurr in $C\}$.
Observe that, being KB finite, there is
an $n\geq 0$ such that, for all $m> n, E_m = E_n$ or $E_m =\emptyset$.
A concept $C$ has {\em rank} $i$ (denoted $\rf(C)=i$) for TBox,
iff $i$ is the least natural number for which $C$ is
not exceptional for $E_{i}$. {If $C$ is exceptional for all
$E_{i}$ then $\rf(C)=\infty$ ($C$ has no rank).}

Rational closure builds on this notion of exceptionality:

\begin{definition}[Rational closure of TBox] \label{def:rational closureDL}
Let $K=( {\cal T}, {\cal A} )$ be an $\alctr$ knowledge base. The
rational closure, $\overline{\mathit{TBox}}$,  of the TBox ${\cal T}$, 
is defined as:

    $\overline{\mathit{TBox}}$=$\{\tip(C) \sqsubseteq D \tc \mbox{either} \ \rf(C) < \rf(C \sqcap \nott D)$ $\mbox{or} \ \rf(C)=\infty\} \ \unione$ 
   
    $\mbox{\ \ \ \ \ \ \ \ \ \ \ \ \ }\{C \sqsubseteq D \tc \ \mbox{KB} \ \models_{\alctr} C \sqsubseteq D\}$,
where $C$ and $D$ are $\alc$ concepts.
\end{definition}


A good property of rational closure is that, for $\alc$,
deciding  if an inclusion $\tip(C) \sqsubseteq D$ belongs to the rational closure of TBox is a problem in \textsc{ExpTime} \cite{AIJ15}.

In \cite{AIJ15} it is shown that the semantics corresponding to rational closure can be given in terms of {\em minimal canonical} $\alctr$ models. With respect to standard $\alctr$ models, in such models the rank of each domain element is as low as possible (each domain element is assumed to be as typical as possible). 
This is expressed by the following definition.

 \begin{definition}[Minimal models of $K$ (with respect to $TBox$)]\label{Preference between models in case of fixed valuation} 
Given $\emme = $$\langle \Delta, <, I \rangle$ and $\emme' =
\langle \Delta', <', I' \rangle$ , we say that $\emme$ is preferred to
$\emme'$ \hide{with respect to the fixed interpretations minimal
semantics} ($\emme \prec \emme'$) if:
$\Delta = \Delta'$,
$C^I = C^{I'}$ for all (non-extended) concepts $C$,
for all $x \in \Delta$, it holds that $ k_{\emme}(x) \leq k_{\emme'}(x)$ whereas
there exists $y \in \Delta$ such that $ k_{\emme}(y) < k_{\emme'}(y)$.

Given a knowledge base $K=( {\cal T}, {\cal A} )$, we say that
$\emme$ is a minimal model of $K$ (with respect to TBox)  if it is a model satisfying $K$ and  there is no
$\emme'$ model satisfying $K$ such that $\emme' < \emme$.
\end{definition}
Furthermore, the models corresponding to rational closure are canonical. This property, expressed by the following definition, is needed when reasoning about the (relative) rank of the concepts: it is important to have them all represented.

%

\begin{definition}[Canonical model\hide{with respect to $\lingconc$}]\label{def-canonical-model-DL}
Given $K=( {\cal T}, {\cal A} )$, 
a  model $\emme=$$\sx \Delta, <, I \dx$ satisfying $K$ is 
{\em canonical} 
 if for each set of concepts
$\{C_1, C_2, \dots, C_n\}$
consistent with $K$, there exists (at least) a domain element $x \in \Delta$ such that
$x \in (C_1 \sqcap C_2 \sqcap \dots \sqcap C_n)^I$. \end{definition}
%
%
%
\begin{definition}[Minimal canonical models (with respect to TBox)]\label{Preference between models wrt TBox}
$\emme$ is a minimal canonical model of $K$, 
if it is a canonical model of $K$ and it is minimal with respect $<$ (see Definition \ref{Preference between models in case of fixed
valuation}) among the canonical models of $K$.
\end{definition}

The correspondence between minimal canonical models and rational closure is established by the following key theorem. 

\begin{theorem}[\cite{AIJ15}]\label{Theorem_RC_TBox}
Let $K=( {\cal T}, {\cal A} )$ be a knowledge base and $C \sqsubseteq D$ a query.
Let $\overline{\mathit{TBox}}$ be the rational closure of $K$ w.r.t. TBox.
We have that $C \sqsubseteq D \in$ $\overline{\mathit{TBox}}$ if and only if $C \sqsubseteq D$ holds in all minimal  canonical models 
of $K$ with respect to TBox. 
\end{theorem}


\section{Enriched Preferential Semantics}\label{sec:W-enriched}

The main weakness of rational closure, despite its power and its nice computational properties, is that it is an all-or-nothing mechanism that does not allow to separately reason on single {\em aspects}. As mentioned in the introduction, to overcome this difficulty, here we consider models with  several preference relations, one for each aspect we want to reason about. We assume an aspect can be  {\em any concept} occurring in $K$ 
  on the right end side of some typicality inclusion $\tip(C) \sqsubseteq A$:
we call ${\cal L_{A}}$ the set of these aspects.
Observe that $A$ may be non-atomic; it is an arbitrary non-extended concept. 
For each aspect $A\in {\cal L_{A}}$, the relation $<_A$ expresses the preference for aspect $A$  being true: $<_{Fly}$ expresses the preference for flying, so if it holds  that $\tip(Bird) \sqsubseteq Fly$, birds that do fly will be preferred  to birds that do not fly, with respect to aspect fly, i.e. with respect to $<_{Fly}$. 

Notice that the preferences with respect to aspects might be conflicting. It can be that, for instance,
$x$ is preferred to $y$ for aspect $A_i$ ($x <_{A_i} y$), whereas $y$ is preferred to $x$ for aspect $A_j$ ($y <_{A_j} x$). 
In the example of birds, we can have that $x <_{Fly} y$, whereas $y <_{HasNiceFeather} x$.

With this semantic richness we aim to obtain a strengthening of rational closure in which typicality with respect to every aspect is maximized.
Since we want to compare our approach to rational closure,  we keep the language the same as in $\alctr$. In particular, we only include a single  typicality operator $\tip$. However, the semantic richness could motivate the introduction of  several typicality  operators $\tip_{A_1} \dots \tip_{A_n}$ by which one could explicitly refer within the language to the typicality w.r.t. aspect $A_1$, or $A_2$, and so on. We leave this extension for future work. 


Let us now enrich the definition of an $\alctr$ model given above (Definition \ref{semalctr}) by taking into account preferences with respect to the aspects, as well as a {\em global} preference relation $<$.  
\begin{definition}[Enriched rational interpretation]\label{def-enrichedmodelR}
An enriched rational interpretation is a structure  $\emme = \langle \Delta, <_{A_1}, \ldots,<_{A_n},<, I \rangle$, where
$\Delta$ and $I$ are a domain and an interpretation function (as in Definition \ref{semalctr}), $<_{A_1}, \ldots,<_{A_n}, <$ are irreflexive, transitive, modular and well-founded preference relations over $\Delta$. 
Furthermore, $<$ satisfies the condition:  
\begin{quote}
(a) {\bf If} there is some $A_i$ such that $x <_{A_i} y$, and there is no $A_j$ such that $y <_{A_j} x$, 
{\bf then} $x < y$.
\end{quote}

Last, we let:
$min_<(S) = \{x \in S$ s.t. there is no $x_1 \in S$ s.t. $x_1 < x\}$ and $(\tip(C))^I = min_<(C^I)$.

\end{definition}
In the semantics above the global preference relation $<$ is related to  the various preference relations $<_{A_i}$, relative to single aspects $A_i$, 
 that we call {\em indexed} preference relations. Given condition (a), $x<y$ holds when   $x$ is preferred to $y$ for a single aspect $A_i$, and there is no aspect $A_j$ for which $y$ is preferred to $x$.
This allows to define preferences among elements having the same rank in the minimal canonical models of the rational closure.
As it will become clear, this brings us towards the direction of a refinement of the semantics of rational closure.

Let $min_{<_{A_i}}(S) = \{x \in S$ s.t. there is no $x_1 \in S$ s.t. $x_1 <_{A_i} x\}$.
In order to be a model of $K$, an enriched rational model must satisfy the following conditions.

\begin{definition}[Enriched rational models of K]\label{enriched-model-of-K}
Given a knowledge base $K=({\cal T}$,{\cal A}$)$,  an {\em enriched rational model} (or enriched model) for $K$ 
is an enriched interpretation $\emme = \langle \Delta, <_{A_1}, \ldots,<_{A_n}, <, I \rangle$ of $K$ which satisfies ${\cal T}$ and ${\cal A}$, where:

$- \; \emme$  satisfies the TBox ${\cal T}$ if 
\begin{itemize}
\item[(1)]
for all strict inclusions $C \sqsubseteq D \in {\cal T}$ 
(where $\tip$ does not occur in $C$), $C^I \subseteq {D}^I$;
\item[(2)]
for all typicality inclusions $\tip(C) \sqsubseteq A_i \in {\cal T}$,  
$min_<({C}^I) \subseteq {A_i}^I$;
\item[(3)]
for all typicality inclusions $\tip(C) \sqsubseteq A_i \in {\cal T}$,  
$min_{<_{A_i}}({C}^I) \subseteq {A_i}^I$.

\end{itemize}

$-\;\emme$ satisfies the ABox ${\cal A}$ if: 
(i) for all $C(a) \in {\cal A}$,  $a^I \in C^I$; (ii) for all $aRb \in {\cal A}$,  $(a^I,b^I) \in R^I$.

\end{definition}
By condition (3), the domain elements satisfying all the defeasible inclusions concerning aspect $A_i$ will be  preferred with respect to $<_{A_i}$ to those falsifying some of them.

We call $\alctre$ the description logic extending $\alc$ with typicality under the enriched semantics.
Logical entailment in $\alctre$  is defined as usual: a query $F$ (with form $C_L(a)$ or $C_L \sqsubseteq C_R$) is {\em logically entailed by $K$}  (written $K \models_{\alctre} F$)  if $F$ holds in all the enriched models of $K$.

The following example shows that, at least in some cases, condition $(a)$ allows to establish the expected preference between individuals.

\begin{example} \label{example-HasNiceFeather}
Let 
 {\small ${\cal T} = \{Penguin \sqsubseteq Bird$,  $\tip(Bird) \sqsubseteq HasNiceFeather$,  $\tip(Bird) \sqsubseteq Fly$,  $\tip(Penguin) \sqsubseteq \neg Fly \}$.
$\elle_A = \{HasNiceFeather, Fly,  \neg Fly, Bird, Penguin\}$}.  
We consider an $\alctre$ model $\emme$ of $K$, that we don't fully describe but which we only use to observe the behavior of two Penguins 
$x$, $y$ with respect to the properties of (not) flying and having nice feather. 
 In particular,  let us consider the three preference relations: $<, <_{Fly}, <_{\neg Fly}, <_{HasNiceFeather}$. 

Suppose  $x <_{HasNiceFeather} y$ (because $x$, as all typical birds, has a nice feather whereas $y$ does not)  and there is no other aspect $A_i$ such that $y <_{A_i} x$, and in particular it does neither hold that  $y <_{\neg Fly} x$ (because for instance, as all typical penguins, both $x$ and $y$ do not fly), nor that $y <_{Fly} x$. In this case, obviously it holds that $x < y$, since condition (a) in Definition \ref{def-enrichedmodelR} is satisfied.

\end{example}

However, the enriched semantics does not provide a refinement of the rational closure.
\begin{example} \label{example-S-enriched}
Let us compare the domain element  $x \in (Bird \sqcap  Penguin \sqcap \neg Fly \sqcap \neg HasNiceFeather)^I$ in a model $\emme$ with a domain element $y \in (Bird \sqcap  Penguin \sqcap Fly \sqcap HasNiceFeather)^I$.
We have that $y <_{Fly} x$,  $y <_{HasNiceFeather} x$ and $x <_{\neg Fly} y$. Hence, condition (a) cannot help to conclude anything about the global relation $<$ concerning $x$ and $y$ (and, in particular, we cannot conclude $x<y$).
However, in all the models of the rational closure, we would prefer $x$ to $y$.
\end{example}
Observe that, in this last example, $x$ violates the defeasible properties of Birds of flying and having a nice feather, while $y$ violates the more specific defeasible property of Penguin of not flying. 


\subsection{S-Enriched models}\label{sec:S-enriched}

In order to deal with cases as Example \ref{example-S-enriched} and in order to obtain a strengthening of rational closure,
we strengthen the definition of enriched model, by introducing an additional condition beside condition (a).
In particular, we define a subset of enriched models, that we call {\em strongly enriched} (S-enriched) models, as they enforce the respect for ``specificity" also in cases when enriched models do not.
In addition to the constraints linking the global preference relation $<$ to the indexed preference relations $<_{A_1} \dots <_{A_n}$, which leads to preferring (with respect to the global $<$) the individuals that are minimal with respect to the aspects $A_i$,
we add a further constraint which leads to prefer the individuals violating defeasible properties of less specific concepts with respect to individuals violating defeasible properties of more specific concepts.
It turns out that this leads to a stronger semantics, 
which is able to capture wanted inferences, such as those in Example \ref{example-S-enriched}, and which provides a strengthening of rational closure semantics in Section \ref{sez:semantica}.

 In order to define S-enriched models of a knowledge base $K$, we strengthen the definition of satisfiability of a TBox as follows.
 
 \begin{definition}[S-enriched rational models of K]\label{def-SenrichedmodelR}
Given a knowledge base $K=({\cal T}, {\cal A})$, an enriched interpretation $\emme = \langle \Delta, <_{A_1}, \ldots,<_{A_n}, <, I \rangle$  is an S-enriched rational model for $K$  if $\emme$ satisfies the TBox ${\cal T}$ and the ABox ${\cal A}$, where:

$ \emme$  satisfies ${\cal T}$ if 
\begin{itemize}
\item[(1)]
for all strict inclusions $C \sqsubseteq D \in {\cal T}$ 
(i.e., $\tip$ does not occur in $C$), $C^I \subseteq {D}^I$;
\item[(2)]
for all typicality inclusions $\tip(C) \sqsubseteq A_i \in {\cal T}$,  
$min_<({C}^I) \subseteq {A_i}^I$;
\item[(3)]
for all typicality inclusions $\tip(C) \sqsubseteq A_i \in {\cal T}$,  
$min_{<_{A_i}}({C}^I) \subseteq {A_i}^I$.

\item[(4)]
{\bf If}  there is $\tip(C_i) \sqsubseteq A_i \in {\cal T}$ s.t. $x <_{A_i} y$ and  $y \in C_i^I$
  and,\\
 for all $\tip(C_j) \sqsubseteq A_j \in {\cal T}$ s.t.  $y <_{A_j} x$ and  $x \in C_j^I$, 
there is $\tip(C_k) \sqsubseteq A_k \in {\cal T}$ s.t.    $x <_{A_k} y$, \  $y \in C_k^I$, 
and $k_{\emme}(C_j) < k_{\emme}(C_k)$, \\
{\bf then} $x < y$.

\end{itemize}

$\emme$ satisfies {\cal A}  if: 
(i) for all $C(a) \in {\cal A}$,  $a^I \in C^I$, (ii) for all $aRb \in {\cal A}$,  $(a^I,b^I) \in R^I$.

\end{definition}
We call $\alctrse$ the logic based on the semantics of S-enriched models and we define logical entailment in $\alctrse$ as usual: a query $F$  is {\em logically entailed by $K$  in $\alctrse$} (written $K \models_{\alctrse} F$)  if $F$ holds in all the S-enriched models of $K$.

We call condition $(4)$ ``specificity condition", as it captures the idea that, in case two individuals are preferred one another with respect to different aspects, preference (with respect to the global preference relation $<$) should be given to the individual that falsifies typical properties of concepts with lower ranks.
Violating a default property of a less specific concept $C_j$ is less serious than violating a default property of a more specific concept $C_k$.

%

%
The idea is that, in S-enriched models, $<$ provides a strengthening of the preference relation in ranked models of rational closure. 
In particular, further preferences are determined among the elements having the same rank in the models of rational closure, both by exploiting the preference relations $<_{A_i}$ with respect to single aspects (condition $(a)$), and by exploiting the specificity criterium (condition $(4)$). 
Observe that  $(4)$ is only a sufficient condition for $x<y$. It is not required to be a necessary condition, 
and  additional  pairs $x'<y'$ might be needed for $<$ to satisfy modularity.


The above semantics allows us to model a form of inheritance in which  the defeasible properties of concepts (classes) are inherited by more specific concepts, unless they are overridden by the properties of more specific ones. Also, the overriding of some defeasible property of a concept should not cause the overriding of all the defeasible properties of that concept, and the inheritance of a more specific property should win over the inheritance of a less specific one.
These criteria are incorporated among the desirable principles considered by Lehmann in \cite{Lehmann95} namely, the ``presumption of typicality", the ``presumption of independence",  ``priority to typicality" and ``respect of specificity", which underly the lexicographic closure definition.
Similar criteria are also at the basis of the nonmonotonic description logic ${\cal DL}^N$ \cite{bonattiAIJ15}, whose definition explicitly uses the notion of overriding. We will provide a detailed comparison with the lexicographic closure and with ${\cal DL}^N$ in Section  \ref{sec:Lexicographic}.

With reference to Example \ref{example-S-enriched}, we can see that, with this notion of S-enriched semantics, 
we are able to give preference to a domain element $x$ which is a penguin that does not fly and has not nice feathers (thus violating the defeasible property
$\tip(Bird) \sqsubseteq HasNiceFeather$ of birds)
with respect to an element $y$
corresponding to penguin which has not nice feathers  but flies, violating the more specific property $\tip(Penguin) \sqsubseteq \neg Fly$ of penguins.


For $\alctrse$ we can prove
the following theorem, showing the relations between $\alctrse$ and $\alctr$ (the extension of $\alc$ with the typicality operator defined in Section  \ref{sez:semantica}). It is a an immediate consequence of the fact that a S-enriched model is a $\alctr$ model.
 \begin{theorem}\label{equivalence-multiple2}
If $K \models_{\alctr} F$ then also $K \models_{\alctrse} F$. If $\tip$ does not occur in $F$ the other direction also holds: If $K \models_{\alctrse} F$ then also $K \models_{\alctr} F$.
 \end{theorem} 
The theorem is a an immediate consequence of the fact that a S-enriched model is a $\alctr$ model. By contraposition,
from the hypothesis that $K \not \models_{\alctrse} F$, there is an S-enriched model $\emme$ satisfying $K$ and falsifying $F$.
Since $\emme$ is also an $\alctr$ model of $K$, it follows that $K \not \models_{\alctr} F$.
For the second part, observe that, by contraposition, if $K \not \models_{\alctr} F$, then there is an $\alctr$ model $\emme= \langle \Delta,<, I \rangle$ of $K$ falsifying $F$. We can define an S-enriched model of $K$, $\emme'=  \langle \Delta, <_{A_1}, \ldots, <_{A_n}, <, I \rangle$,
by letting, for all $i=1,\ldots,n$, $<_{A_i}= <$. It is easy to see that $\emme'$ satisfies condition (a) in Definition \ref{def-enrichedmodelR}
as well as conditions (1)-(4) in Definition \ref{def-SenrichedmodelR} and, hence, it is an S-enriched model of $K$ which falsifies $F$.

\section{Minimal S-enriched models and their relation with rational closure} \label{sec:minimal-S-enriched-models}

As in the semantic characterization of the rational closure in Section \ref{sez:semantica}, we restrict our consideration to {\em minimal canonical models} of the KB.
We define minimal S-enriched models by first minimizing the rank of each domain element with respect to the  indexed preference relations $<_{A_i}$'s, and then by minimizing the rank of the elements with respect to the global preference relation $<$. 
Let $k_{{\emme,A_i}}(x)$ be the rank of a domain element $x$ of the model $\emme$ with respect to the indexed relation $<_{A_i}$.
%

\begin{definition}[Minimal S-enriched models of $K$ (with respect to $TBox$)]\label{minimal_enriched_models_new} 
Given two S-enriched models $\emme = $$\langle \Delta, <_{A_1}, \ldots,<_{A_n}, <, I \rangle$ and $\emme' =
\langle \Delta', <'_{A_1}, \ldots,<'_{A_n}, <', I' \rangle$, 
\begin{itemize}
\item
 $\emme'$ is preferred to
$\emme$  {\em w.r.t. the aspects} (and write $\emme' \prec_{Aspects} \emme$) if 
$\Delta = \Delta'$, $I = I'$, and:
\begin{itemize}
\item  for all $x \in \Delta$, $ k_{{\emme',A_i}}(x) \leq k_{{\emme,A_i}}(x)$;
\item for some $y \in \Delta$, $ k_{{\emme',A_i}}(y) < k_{{\emme,A_i}}(y)$  
\end{itemize} 

\item  $\emme'$ is preferred to $\emme$ {\em  w.r.t. the global preference relation $<$ } (and write $\emme' \prec_{global} \emme$)
 if $\Delta = \Delta'$, $I = I'$, and
 \begin{itemize}
\item   for all $x \in \Delta$, $ k_{{\emme'}}(x) \leq k_{{\emme}}(x)$;
\item for some $y \in \Delta$ ,  $ k_{{\emme'}}(y) < k_{{\emme}}(y)$ 
\end{itemize}  
\end{itemize}

We combine the two preference relations in the {\em lexicographic order}:
 We  say that  {\em $\emme'$ is preferred to $\emme$} (and write $\emme' \prec \emme$)
 if $\Delta = \Delta'$, $I = I'$, and
  \begin{itemize}
  \item either  $\emme' \prec_{Aspects} \emme$ or
  \item    $\emme \not \prec_{Aspects} \emme'$
  and $\emme' \prec_{global} \emme$.
\end{itemize}   

Given a knowledge base $K = \langle {\cal T}, {\cal A} \rangle$, we say that
$\emme$ is a {\em minimal S-enriched model of $K$} (with respect to TBox)  if it is an S-enriched model of $K$ and  there is no model
$\emme'$ satisfying $K$ such that $\emme' \prec \emme$.

\end{definition}
%
As the definition of the global preference $<$ depends on the indexed preferences $A_i$, we first minimize with respects to the aspects $A_i$ 
and, then, with respect to the  global preference $<$. 

In minimal models, each preference relation $<_{A_i}$  
ranks the domain elements into two levels: 
the domain elements $x$ satisfying all the defeasible inclusion concerning aspect $A_i$ (having rank $k_{\emme,A_i}(x)=0$), and the domain elements $y$ falsifying some defeasible inclusion concerning aspect $A_i$ (having rank $k_{\emme,A_i}(y)>0$).  This is similar to the interpretation given by Lehmann to single defaults in \cite{Lehmann95}). \normalcolor

Let us restrict our attention to minimal S-enriched  models which are {\em canonical}. 
\begin{definition}[Minimal canonical S-enriched models of K]\label{Def-minimal-enriched-canonical}
A {\em minimal canonical S-enriched model  $\emme$ of $K$ }
is an S-enriched model of $K$, which is  minimal (with respect to Definition \ref{minimal_enriched_models_new})
and it is canonical, i.e., for each set of (non-extended) concepts $\{C_1, C_2, \dots,$ $ C_n\}$\hide{ \subseteq \lingconc$} s.t. $K \not\models_{\alctrse} C_1 \sqcap C_2 \sqcap \dots \sqcap C_n \sqsubseteq \bot$, there exists (at least) a domain element $x$ such that
$x \in (C_1 \sqcap C_2 \sqcap \dots \sqcap C_n)^I$. 
\end{definition}

In the following we will write: $K {\small \models^{min}_{\alctrse}}  C \sqsubseteq D$ to mean that $C \sqsubseteq D$ holds in all {\em minimal canonical} S-enriched models of $K$.

The following example shows that this semantics allows us to correctly deal with the wanted inferences.
and, in particular, that inheritance of defeasible properties, when not overridden for more specific concepts, applies to concepts of all ranks.

\begin{example} \label{exa:BabyPenguin} 
Consider a knowledge base $K$=({\cal T} ,{\cal A}), where ${\cal A}=\emptyset$ and  ${\cal T}$ contains the following inclusions: 

$\mathit{\tip(Bird) \sqsubseteq  Fly}$

$\mathit{Penguin \sqsubseteq Bird}$

$\mathit{\tip(Penguin) \sqsubseteq \neg Fly}$

$\mathit{\tip(Penguin) \sqsubseteq BlackFeather}$ 

$\mathit{BabyPenguin \sqsubseteq Penguin}$

$\mathit{\tip(BabyPenguin)}$ $ \sqsubseteq$ $\mathit{ \neg BlackFeather}$.

\noindent
As we have seen from Example \ref{example-HasNiceFeather}, the defeasible property of birds having a nice feather is inherited by typical penguins,  even though penguins are exceptional birds regarding flying.
Here, we also expect that typical baby penguins inherit the defeasible property of penguins that they do not fly (by presumption of independence \cite{Lehmann95}), although the defeasible property $\mathit{BlackFeather}$ is instead overridden for typical baby penguins. 

Consider two domain elements $z$ and $w$ which are both baby penguins and have a non black feather.
Suppose that $z$ flies and $w$ doesn't. Then $z$ violates the defeasible property that penguins typically do not fly,  while $w$ violates the defeasible property   that birds typically fly. As $z<_{Fly} w$ and $w <_{\neg Fly} z$, condition $(a)$ neither allows to conclude  $w<z$, nor $z<w$. However, $z$ violates a more specific defeasible property than $w$ and, hence, by the specificity condition $(4)$ of S-enriched models in Definition \ref{def-SenrichedmodelR}, we can conclude that $w < z$ holds. Indeed, the S-enriched minimal model semantics allows us to conclude that $\tip(BabyPenguin) \sqsubseteq \neg Fly$,
as wanted.
\end{example}

\normalcolor

We have developed the semantics above in order to overcome a weakness of rational closure, namely its all-or-nothing character. 
In order to show that the semantics hits the point, 
we prove that the semantics of minimal canonical S-enriched models is a refinement of the semantics of rational closure, i.e. that minimal entailment in $\alctrse$ strengthens  reasoning under the rational closure. 
\normalcolor


\begin{theorem}\label{stronger_semantics} 
Let $K=(TBox, ABox)$ be a knowledge base. If $C \sqsubseteq D \in \overline{TBox}$ then $K {\small \models^{min}_{\alctrse}} C \sqsubseteq D$.
 \end{theorem} 
 \begin{proof} 
%
%
By contraposition suppose that $K \not\models^{min}_{\alctrse} C \sqsubseteq D$. 
Then there is a minimal canonical S-enriched  model  $\emme =  \langle \Delta,  <_{A_1}, \ldots,<_{A_n}, <, I  \rangle$  of $K$ and an $y \in C^I$ such that $y \not\in D^I$. 
All sets of concepts consistent with $K$ w.r.t. $\alctrse$ are also consistent with $K$ with respect to $\alctr$, and viceversa (by Theorem \ref{equivalence-multiple2}). 
 By definition of {\em canonical}, $\emme'= \langle {\Delta},  <, I  \rangle$ is a canonical $\alctr$ model of $K$ according to Definition \ref{semalctr} in Section \ref{sez:semantica}. Also, there must be a minimal canonical model of $K$ obtained from $\emme'$ by possibly lowering the ranks of domain elements. Let $\emme_{RC} =  \langle {\Delta},  <_{RC}, I  \rangle$ be such a model.

If $C$ does not contain the $\tip$ operator, we are done: in $\emme_{RC}$, as in $\emme$, there is $y \in C^{I}$ such that $y \not\in D^{I}$, hence $C \sqsubseteq D$ does not hold in  $\emme_{RC}$, and $C \sqsubseteq D \not\in \overline{TBox}$.
If $\tip$ occurs in $C$, and $C = \tip(C')$, we still need to show that also in $\emme_{RC}$, as in $\emme$,  $y \in (\tip(C))^I$, i.e.
$y \in min_{<_{RC}}({C'}^I)$.  We prove this by showing that for all $x,y \in \Delta$ if $x <_{RC} y$ in $\emme_{RC}$, then also 
$x <y$ in $\emme$. 
%
%
The proof is by induction on  $k_{\emme_{RC}}(x)$.
 
For the base case, let $k_{\emme_{RC}}(x)= 0$ and $k_{\emme_{RC}}(y)>0$. Since $x$ does not violate any inclusion, also in $\emme$ (by minimality of $\emme$)  $k_{\emme}(x)= 0$. This cannot hold for $y$, for which $k_{\emme}(y)> 0$ (otherwise $\emme$ would violate 
$K$, against the hypothesis). Hence $x < y$ holds in $\emme$.

For the inductive case, let $k_{\emme_{RC}}(x)= i < k_{\emme_{RC}}(y)$, i.e. $x<_{RC} y$.
As $x <_{RC} y$ in $\emme_{RC}$ and the rank of $x$ in $\emme_{RC}$ is $i$,
there must be a $\tip(B_i) \sqsubseteq A_i \in E_i - E_{i+1}$ such that $x \in (\neg B_i \sqcup A_i)^I$ whereas $y \in (B_i \sqcap \neg A_i)^I$ in $\emme_{RC}$, so that  $x<_{A_i} y$ holds in the minimal S-enriched model $\emme$. 

Before we proceed let us notice that by definition of $E_i$ in Section \ref{sez:semantica}, as well as by what stated just above on the relation between rank of a concept and $k_{\emme_{RC}}$, $k_{\emme_{RC}}(B_i)= k_{\emme_{RC}}(x)$. We will use this fact below.
%
%
We show that, for any inclusion $\tip(B_l) \sqsubseteq A_l \in K$  such that $y <_{A_l} x$ and $x \in B_l$,
it holds that $k_{\emme}(B_l) < k_{\emme}(B_i)$, so that, by (4), $x<y$.

Let $\tip(B_l) \sqsubseteq A_l \in K$
such that $y <_{A_l} x$ and $x \in B_l$.
As $\emme$ is a minimal model, $\tip(B_l) \sqsubseteq A_l \in K$ is
 violated by $x$, i.e.  $x \in (B_l \sqcap \neg A_l)^I$.
Since $\emme_{RC}$ satisfies $K$,  $x$ cannot be a typical $B_l$-element and  there must be $x' <_{RC} x$ in $\emme_{RC}$ with $x' \in (B_l)^I$. 
As $k_{\emme_{RC}}(x')< i$, 
by inductive hypothesis,  $x' < x$ in $\emme$. 
As $x' \in {B_l}^I$, $k_{\emme}(B_l) \leq k_{\emme}(x')$. 
Since it can be shown that $k_{\emme}(x') < k_{\emme}(B_i) $, $k_{\emme}(B_l) < k_{\emme}(B_i)$, and by condition (4), it holds that $x < y$ in $\emme$. 

%
With these facts,  since $y \in min_<({C'}^I)$  holds in $\emme$, also $y \in min_{<_{RC}}({C'}^{I})$ in $\emme_{RC}$, hence $\tip(C') \sqsubseteq D$ does not hold in $\emme_{RC}$, and $C \sqsubseteq D = \tip(C') \sqsubseteq D \not\in \overline{TBox}$.

The theorem follows by contraposition.
\hfill \qed
\end{proof}

Observe that, in the proof of Theorem \ref{stronger_semantics}, we have not used condition (a) of Definition \ref{enriched-model-of-K}.
Indeed, we can show that the specificity condition $(4)$ in minimal S-enriched models (Definition \ref{def-SenrichedmodelR}) subsumes condition (a), dealing with multiple aspects. 
Let us consider a simplified notion of S-enriched model in which condition $(a)$ is omitted. 

\begin{proposition}\label{elimino(a)}
Let $\emme= \langle \Delta,  <_{A_1}, \ldots,<_{A_n}, <, I \rangle$ be a minimal simplified S-enriched model of $K$. We can show that if condition 
(4) holds, then condition (a) holds as well.
\end{proposition}
\begin{proof}
To see that condition (4) implies condition (a), suppose that the precondition of $(a)$ holds, i.e., that
there is some $A_i$ such that $x <_{A_i} y$ in $\emme$, and there is no $A_j$ such that $y <_{A_j} x$.
We show that $x<y$ follows using condition (4).

As $x <_{A_i} y$ and the model $\emme$ is minimal, in particular, it is minimal with respect to the aspects and there must be a defeasible inclusion 
$\tip(C) \sqsubseteq A_i \in K$ s.t.  $x$ satisfies it  ($x  \in (\neg C \sqcup  A_i)^I$), and $y$ violates it ($y \in (C \sqcap \neg A_i)^I$).
Additionally, for all $A_j$ ($i\neq j$), $y \not <_{A_j} x$, that is, all the defeasible inclusions satisfied by $y$ are also satisfied by $x$.
Therefore, the antecedent of condition (4), the ``If part", holds as 
there is no inclusion which is falsified by $x$ and satisfied by $y$. Hence, by condition (4), $x<y$ follows.
\hfill \qed
\end{proof}
\normalcolor

\section{The multipreference-closure} \label{sec:closure}

As the minimal S-enriched semantics is a strengthening of the rational closure semantics, 
in this section we build on the rational closure to define a new notion of closure, that we call the {\em multipreference-closure} 
(MP-closure, for short).
We show that the MP-closure provides a sound approximation of the minimal S-enriched semantics: 
reasoning in the MP-closure will allow to derive sound conclusions with respect to the minimal canonical S-enriched models semantics.
The MP-closure can be regarded as a variant of the lexicographic closure \cite{Lehmann95} and
we compare the MP-closure  with the lexicographic closure and with the nonmonotonic description logic ${\cal DL}^N$\cite{bonattiAIJ15}.


According to condition $(4)$ in Definition \ref{def-SenrichedmodelR}, the rank of concepts in a S-enriched model is used to determine the specificity of typicality inclusions and, thus, to determine the preference relation $<$ among domain elements. 
As the ranking of concepts in the rational closure approximates the ranking of concepts in minimal canonical S-enriched models of the KB,
it can be used for determining the specificity of typicality inclusions in a closure thus providing a sound approximation of minimal entailment
in the S-enriched semantics. 

In particular, if $\rf(B)$ is the rank of a concept $B$ in the rational closure,
the most preferred $B$-elements in minimal canonical S-enriched models must be among the $B$-elements with rank $\rf(B)$.
According to condition $(4)$ we prefer a $B$-element $x$ with rank $\rf(B)$ to another one $y$ with the same rank, if 
for all the defeasible inclusions falsified by $x$ and not by $y$ there is a more specific defeasible inclusion falsified by $y$ and not by $x$.
In essence, we need to identify those $B$-elements with rank $\rf(B)$ which satisfy a maximal subset of defeasible inclusions, containing inclusions being as specific as possible 
(something which is very similar to what the lexicographic closure construction \cite{Lehmann95} does).

In the following, we provide a construction (the MP-closure) to check the entailment of a subsumption query $\tip(B) \sqsubseteq D$
from a TBox 
%
and we show that the logical consequences under the MP-closure are sound with respect to the S-enriched semantics.
Given a TBox ${\cal T}$, we compute the sequence of TBoxes
 $E_0,E_1, \ldots, E_n$ according to the rational closure construction in Section \ref{sez:semantica}.
 We let  $\delta(E_i)$ be the set of typicality inclusions contained in $E_i$ (i.e. those defeasible inclusions with rank $\geq i$) and
 let $D_i= ,\delta(E_i) -\delta(E_{i+1})$  be the set of typicality inclusions with rank $i$.
Observe that $\delta(E_0)= \delta({\cal T})$.
Given a set $S$ of typicality inclusions,  we let: $S_i= S \cap D_i$, for all ranks $i=0,\ldots, n$ in the rational closure, 
thus defining a partition of the set $S$ according to the rank. 
We introduce a preference relation among sets of typicality inclusions as follows:
$S' \prec S$ ({\em $S'$ is preferred to $S$})  if and only if
there is an $h$ such that, $S_h\subset S'_h$ and, for all $j > h$, $S'_j=S_j$.
The meaning of $S' \prec S$ is that, considering the highest rank $h$ in which $S$ and $S'$ do not contain the same defeasible inclusions,
 $S'$ contains more defeasible inclusions in $D_h$ than $S$.

 \begin{definition}
Let $B$ be a concept such that $\rf(B)=k$ and
let $S \subseteq \delta(TBox)$.
$ S  \cup E_k $ is a  {\em maximal set of defeasible inclusions compatible with $B$ in $K$}  
if:
\begin{itemize}
\item
$ E_k \not \models_{\alctr} \tip(\top) \cap \tilde{S} \sqsubseteq \neg B$
\item
and there is no $S' \subseteq \delta(TBox)$
such that 
$E_k \not \models_{\alctr} \tip(\top) \cap \tilde{S'} \sqsubseteq \neg B$ and $S' \prec S$ ($S'$ is preferred to $S$).
\end{itemize}
where $\tilde{S}$ is the materialization of $S$, i.e.,  $\tilde{S}= {\sqcap} \{ (\neg C \sqcup D) \mid  \tip(C) \sqsubseteq D \in S \}$.

\end{definition}
Informally, $S$ is a maximal set of defeasible inclusions compatible with $B$ and $E_k$ if
there is no set $S'$ which is consistent with $E_k$ and $B$ and is preferred to $S$ since it contains more specific defeasible inclusions.
The construction is  similar to that of the lexicographic closure \cite{Lehmann95,Casinistraccia2012}, although, in this case, 
the comparison of the sets of defeasible inclusions with the same rank (i.e. of $S_i$ and $S'_i$) is based on subset inclusion rather than on the size of the sets, as in the lexicographic closure.

To check if a subsumption $\tip(B) \sqsubseteq D$ is derivable from the MP-closure of TBox 
we have to consider all the maximal sets of defeasible inclusions  $S$ that are compatible with  $B$. 
\begin{definition}
Let  $\tip(B) \sqsubseteq D$ be a query and let $k=\rf(B)$ be the rank of concept $B$ in the rational closure.
$\tip(B) \sqsubseteq D$ {\em follows from the MP-closure of TBox} if  
for all the maximal sets of defeasible inclusions   $S$ that are compatible with $B$ in $K$, 
we have:
$$E_k \models_{\alctr} \tip(\top) \sqcap \tilde{S} \sqsubseteq (\neg B \sqcup D)$$
\end{definition}

\normalcolor
Verifying whether a query $\tip(B) \sqsubseteq D$ 
is derivable from the MP-closure of the TBox
in the worst case requires to consider an exponential number (in the number of typicality inclusions in $K$) of
maximal subsets $S$ of defeasible inclusions compatible with  $B$ and $E_k$.
As entailment in $\alctr$ can be computed in  \textsc{ExpTime} \cite{AIJ15}, this complexity is still in  \textsc{ExpTime}.
However, in practice, it is clearly less effective than computing subsumption in the rational closure of TBox,
which only requires a polynomial number of calls to  entailment  checks in $\alctr$, which can be computed by a linear encoding of an $\alctr$ KB into $\alc$ \cite{ISMIS2015}.
Let us consider again the knowledge base in Example \ref{example-HasNiceFeather}.
\begin{example}
Let $K$=({\cal T},{\cal A}), where ${\cal T} = \{Penguin \sqsubseteq Bird, \tip(Bird) \sqsubseteq HasNice$- $Feather$, $\tip(Bird) $ $\sqsubseteq Fly$,  $\tip(Penguin) \sqsubseteq \neg Fly \}$ and ${\cal A}= \emptyset$.
We want to check whether the query: $\tip(Penguin) \sqsubseteq HasNiceFeather$ holds in all minimal canonical MP models of the TBox.
From the rational closure of TBox, we have: $\rf(Bird)=0$,  $\rf(Penguin)=1$ and
\begin{quote}
 $E_0 = \{Penguin \sqsubseteq Bird$, $\tip(Bird) \sqsubseteq Fly, \tip(Bird) \sqsubseteq HasNiceFeather \}$ \\
 $E_1=  \{Penguin \sqsubseteq Bird$,\; $\tip(Penguin) \sqsubseteq \neg Fly \}$
  \end{quote}
Let $S= \{ \; \tip(Penguin) \sqsubseteq \neg Fly, \; \tip(Bird) \sqsubseteq HasNiceFeather \}$.
We have:
$$E_1 \not \models_{\alctr} \tip(\top) \sqcap \tilde{S} \sqsubseteq \neg Penguin$$  
and  $S$ is the unique maximal set of defeasible inclusions compatible with $Penguin$.
As it holds that $E_1 \models \tip(\top) \cap \tilde{S} \sqsubseteq (\neg$ $ Penguin \sqcup HasNiceFeather)$,
then $\tip(Penguin)$ $ \sqsubseteq HasNiceFeather $ is derivable from the MP-closure of the TBox,
which is in agreement with the fact that, in all the minimal S-enriched canonical models of the TBox, the typical penguins have a nice feather. $\hfill \bbox$
\end{example} 

It is easy to see that, in the general case, there may be more then one maximal set of typicality inclusions compatible with a given concept $B$.
Coinsider the following example: 
\begin{example} \label{exa:estensioni_multiple}
Let $K$=(TBox,ABox), where $TBox = \{Penguin \sqsubseteq Bird,\; Penguin \sqcap A \sqcap H \sqsubseteq \bot, \; A \sqsubseteq C,$ 
$ H \sqsubseteq C, \;
\tip(Bird) \sqsubseteq Fly, \; \tip(Bird) \sqsubseteq H , \;  \tip(Bird) \sqsubseteq A,  \; \tip(Penguin) \sqsubseteq \neg Fly \}$.

Observe that typical birds have both the properties $H$ and $A$. However, by the second inclusion in TBox,
a typical penguin cannot inherit both property $A$ and property $H$.
It can inherit just one of them and, semantically, we prefer  penguins having either property $A$ or
property $H$ to the penguins that neither have property $A$ nor have property $H$.
Hence, we can conclude that typical penguins have the property $C$ in all the minimal S-enriched canonical models of the KB.
Let $\mbox{TBox}_{Strict}$ be the strict inclusion in the TBox.
 
Given the query $\tip(Penguin) \sqsubseteq C $,
as before $\rf(Bird)=0$, $\rf(Penguin)=1$ and $E_1= \mbox{TBox}_{Strict}  \cup \{ \tip(Penguin) \sqsubseteq \neg Fly \}$.
We have two maximal sets of defeasible inclusions compatible with the concept $Penguin$: 
\begin{quote}
$S = \{ \tip(Bird) \sqsubseteq H, \;  \tip(Penguin) \sqsubseteq \neg Fly \} $  and \\
$S' =  \{ \tip(Bird) \sqsubseteq A, \;  \tip(Penguin) \sqsubseteq \neg Fly \} $, 
 \end{quote}
 The subsumption $\tip(Penguin) \sqsubseteq C $ is derivable from the MP-closure of TBox
(although neither $\tip(Penguin) \sqsubseteq A $ nor $\tip(Penguin) \sqsubseteq H $
can be derived from the MP-closure of TBox), as we have:
\begin{quote}
$E_1 \models \tip(\top) \sqcap \tilde{S} \sqsubseteq (\neg$ $ Penguin \sqcup C)$ and \\
$E_1 \models \tip(\top) \sqcap \tilde{S'} \sqsubseteq (\neg$ $ Penguin \sqcup C)$.
\end{quote}
$S$ characterizes the typical penguins having the default property  $H$ of birds, while 
$S'$ characterizes the typical  penguins having the default property $A$ of birds.
Observe that $\tip(Penguin) \sqsubseteq C $ is not derivable from 
the rational closure for $\alc$ 
recalled in Section 3,
as the rational closure is weaker then the MP-closure. 
\end{example} 

We show the soundness of the MP-closure construction by proving that the typicality inclusions which follow from the MP-closure of a TBox
hold in all the minimal canonical S-enriched models of the TBox.
\begin{proposition} \label{prop:soundnessMP-closure}
If $\tip(B) \sqsubseteq D$ follows from the MP-closure of $K$, then $K \models^{min}_{\alctrse} \tip(B) \sqsubseteq D$.
\end{proposition}

\begin{proof}
By contraposition.

Let $K$ be a knowledge base and $B$ a concept with $\rf(B)=k$ in the rational closure. 
Assume that for some minimal canonical S-enriched model $\emme = \langle \Delta,<_{A_1}, \ldots, <_{A_n}, <, I \rangle$ of $K$ there is an element $x \in \Delta$ such that $x \in min_<(B^I)$ and $x \not \in D^I$.
We prove that there is a maximal set of defeasible inclusions $S$ compatible with $B$ in $K$, such that 
$$E_k \not \models_{\alctr} \tip(\top) \sqcap \tilde{S} \sqsubseteq (\neg B \sqcup D)$$
i.e., $\tip(B) \sqsubseteq D$ does not follow from the MP-closure of $K$.

Let us define $S$ as the set of all the defeasible inclusions in TBox which are satisfied by $x$ in $\emme$, i.e.
$S = \{ \tip(C) \sqsubseteq E \in \mbox{TBox} \mid \; x \in ( \neg C \sqcup E)^I\}$.
We show that $ E_k \not \models_{\alctr} \tip(\top) \sqcap \tilde{S} \sqsubseteq (\neg B \sqcup D)$.

By Theorem \ref{stronger_semantics}, $\emme$ is the refinement of a minimal canonical rational model of $K$ 
 $\emme^{RC}= \langle \Delta,<_{rc}, I \rangle$.
It is easy to see that  $\emme^{RC}$ is a minimal canonical model of the rational closure of $K$.
By a property of $\emme^{RC}$ (Proposition 12 in \cite{AIJ15}), $\emme_k^{RC}$ (i.e. the model obtained by $\emme$ by collapsing all the element with rank $\leq k$ to rank $0$) satisfies $E_k$: $\emme_k^{RC} \models_{\alctr} E_k$.
Also, as $\rf(B)=k$ and  $x \in \tip(B)^I$, $x$ must have rank $k$ in $\emme^{RC}$, and rank $0$ in $\emme_k^{RC}$
(and, clearly, $k_{\emme,rc}(x)=k$ in $\emme$).  
 Thus, $x \in \tip(\top)^I$, but also $x \in (B \sqcap \tilde{S})^I$, therefore
$\emme_k^{RC} \not \models_{\alctr} \tip(\top) \sqcap \tilde{S} \sqsubseteq \neg B$.
Hence,
 $E_k \not \models_{\alctr} \tip(\top)  \sqcap \tilde{S}  \sqsubseteq \neg B$,
 i.e. $S$ is a set of defeasible inclusions compatible with $B$.
 
 Furthermore, as $x \in (\neg D)^I$,  $E_k \not \models_{\alctr} \tip(\top)  \sqcap \tilde{S} \sqcap \neg D \sqsubseteq \neg B$,
 and hence, $ E_k \not \models_{\alctr} \tip(\top) \sqcap \tilde{S} \sqsubseteq (\neg B \sqcup D)$, i.e.,
 $\tip(B) \sqsubseteq D$ does not follow from the MP-closure of TBox.
 
To show that $S$ is a maximal set of defeasible inclusions compatible with $B$,
we have still to show that $S$ is maximal.
Suppose, by contradiction, it is not. Then there is a set $S'$ such that $S' \prec S$ and  
$E_k \not \models_{\alctr} \tip(\top)  \sqcap \tilde{S'}  \sqsubseteq \neg B$.
Therefore, there must be a $\alctr$ model $\enne= \langle \Delta',<'_{rc} I' \rangle$ of  $E_k$ and an element $y \in \Delta'$, having rank $0$ in $\enne$ such that:
$y \in (\tilde{S'} \sqcap B)^{I'}$. 

As $\emme$ is canonical, then $\emme^{RC}$ is canonical as well. Hence, there must be an element $z \in \Delta$ such that $z \in (\tilde{S'} \sqcap B)^I$
(the interpretation of all non-extended concepts in $z$ is the same as in $y$ in $\enne$).
As $y$ has rank $0$ in $\enne$, $y$ satisfies all the defeasible inclusions in $E_k$.
Hence, the concept $\tilde{S'} \sqcap B$ must have rank $k$ in the rational closure and, therefore, $z$ must have rank $k$ in $\emme^{RC}$.
Thus, $z \in (\tip(\top) \sqcap \tilde{S'} \sqcap B)^I$  in $\emme^{RC}$,
and $z$ is as well an element of $\emme$ satisfying $\tilde{S'} \sqcap B$.

Since $S' \prec S$ there must be some 
$h$ such that, $S_h\subset S'_h$ and, for all $j > h$, $S'_j=S_j$.
Thus, there is some defeasible inclusion $\tip(C') \sqsubseteq E' \in S'$ such that $\tip(C') \sqsubseteq E' \not \in S$.
so that $z$ satisfies $\tip(C') \sqsubseteq E'$ (i.e., $z \in (\neg C' \sqcup  E')^I$, while $x$ violates it (i.e., $x \in (C' \sqcap \neg E')^I$).
On the other hand, all the defeasible inclusion violated by $z$ and not by $x$ cannot have rank $\geq h$, as $x$ satisfies 
only the inclusions $S$ (by construction of $S$) and $S'\prec S$. Therefore, $z <y$ holds by condition (4), and $y$ cannot be a typical $B$ element, thus contradicting the hypothesis.
\hfill \qed
\end{proof}

The converse of Proposition  \ref{prop:soundnessMP-closure} does not hold,
as the MP-closure is not complete for minimal entailment in the S-enriched semantics.
In fact, it may occur that, in some minimal canonical S-enriched model $\emme$, there are two concepts $C_j$ and $C_k$,  
such that $k_{\emme}(C_j) < k_{\emme}(C_k)$ although  in the rational closure $\rf(C_j) = \rf(C_k)$.
Consider the following example.

Consider, for instance, the domain description of  the {\em baby penguin} Example \ref{exa:BabyPenguin}.
In any minimal canonical S-enriched model 
 $k_\emme(BP \sqcap \neg Fly) < k_\emme(BP \sqcap Fly)$. However,  in the rational closure,
$rank(BP \sqcap \neg Fly) = rank(BP \sqcap  Fly)= 2$.

Hence, given the two additional defeasible inclusions $d1=\tip(BP \sqcap \neg Fly) \sqsubseteq H$ and
$d2=\tip(BP \sqcap Fly) \sqsubseteq \neg H$,
 in the S-enriched semantics, by condition (4), we would give a higher preference to $d1$ than to $d2$.
This might cause additional conclusions in the S-enriched semantics with respect to those we get from the rational closure.

Clearly, as the two concepts $BP \sqcap \neg Fly$ and $BP \sqcap Fly$ are disjoint, 
it is not the case that a more specific concept which might inherit both $d1$ and $d2$ (with preference to $d1$), 
but a more complex counterexample to the completeness of MP-closure could be built based on this idea.

\subsection{Comparisons with the  Lexicographic Closure and with ${\cal DL}^N$} \label{sec:Lexicographic}

While in the above examples the multipreference-closure gives the same result as lexicographic closure \cite{Lehmann95,Casinistraccia2012}, this is not the case in general.
Indeed, in the lexicographic closure, when there are several defaults having the same ``degree of seriousness" (the same rank),
a preference is given to situations that violate less defaults with respect to situations violating more defaults, considering the cardinality of the sets of violated defaults.  
In our construction, instead, we do not have a preference for accepting  more defaults  
rather than less, if they all have the same rank.

%

To show that in the MP-closure there is no preference for accepting  two defeasible inclusions 
rather than one (if they all have the same rank),
let us consider the following variant of Example \ref{exa:estensioni_multiple}. 

\begin{example} \label{exa:estensioni_multiple_variante}
Let $K$=(TBox,ABox), where TBox = $\{Penguin \sqsubseteq Bird,\; Penguin \sqsubseteq ((\neg A \sqcap \neg H) \sqcup \neg B) , \; A \sqsubseteq C $, \;
$\tip(Bird) \sqsubseteq Fly, \; \tip(Penguin) \sqsubseteq \neg Fly, $ $ \; \tip(Bird) \sqsubseteq H , \;  \tip(Bird) \sqsubseteq A, \; \tip(Bird) \sqsubseteq B \}$
and ABox=$\emptyset$.
There are two maximal sets of defeasible inclusions compatible with the concept $Penguin$, namely: 
\begin{quote}
$S = \{ \tip(Bird) \sqsubseteq H, \; \tip(Bird) \sqsubseteq A, \;  \tip(Penguin) \sqsubseteq \neg Fly \} $  and \\
$S' =  \{ \tip(Bird) \sqsubseteq B, \;  \tip(Penguin) \sqsubseteq \neg Fly \} $, 
 \end{quote}
The subsumption $\tip(Penguin) \sqsubseteq C $ is not derivable from the MP-closure of $K$
as
\begin{quote}
$ E_1 \models \tip(\top) \sqcap \tilde{S} \sqsubseteq (\neg$ $ Penguin \sqcup C)$, but \\
$ E_1 \not\models \tip(\top) \sqcap \tilde{S'} \sqsubseteq (\neg$ $ Penguin \sqcup C)$. 
\end{quote}
Differently from the lexicographic closure, the MP-closure (as the minimal canonical S-enriched models semantics) does not consider $S$ to be preferable to $S'$,
although 
 $S$ contains two defeasible inclusions of rank $0$ while $S'$ only one (and both contain the same number of inclusions of rank $1$).
$\hfill \bbox$
\end{example}

To compare with the nonmonotonic logic ${\cal DL}^N$ introduced in \cite{bonattiAIJ15}, let us now consider the following reformulation of Example \ref{exa:estensioni_multiple} in ${\cal DL}^N$.
In particular, let $KB= ({\cal D}, {\cal S})$, where:
\begin{quote}
${\cal D} =\{ Bird \sqsubseteq_n Fly, \; Penguin \sqsubseteq_n \neg Fly, \; Bird \sqsubseteq_n H, \; Bird \sqsubseteq_n A  \}$ and \\
$ {\cal S} = \{Penguin \sqsubseteq Bird,\; Penguin \sqcap A \sqcap H \sqsubseteq \bot, \; A \sqsubseteq C,\;  H \sqsubseteq C \} $
\end{quote} 
(${\cal D}$ are the defeasible inclusions and ${\cal S}$ the strict ones).

A  consequence of a KB in ${\cal DL}^N$ is that $NPenguin \sqsubseteq \bot$ (where $\mathit{NPenguin}$ represents the prototypical penguin). 
In fact, while the default $\mathit{Bird \sqsubseteq_n Fly}$ is overridden for normal penguins by the  more specific 
defeasible inclusion $\mathit{Penguin \sqsubseteq_n \neg Fly}$,
the defeasible inclusions $\mathit{Bird \sqsubseteq_n H}$ and $\mathit{Bird \sqsubseteq_n A}$ are not overridden by any other properties of normal penguins.
and, as they cannot be both satisfied by normal penguins, the conclusion is that there cannot be normal penguins at all 
(i.e. the prototype $\mathit{NPenguin}$ is inconsistent),
although there may be penguins.
Clearly, in this case, anything can be concluded about normal penguins  (including $\mathit{NPenguin} \sqsubseteq C$),
and the approach in \cite{bonattiAIJ15} would require a repair the knowledge base by a knowledge engineer as there are conflicting defeasible inclusions
among which the conflict cannot be automatically removed.

With the MP-closure we have to consider two alternative scenarios (the two maximal sets $S$ and $S'$ of defeasible inclusions compatible with the concept $\mathit{Penguin}$). 
As in ${\cal DL}^N$, we cannot conclude that typical penguins are $A$s nor that typical penguins are $B$s. 
However, as a difference with respect to ${\cal DL}^N$, we do not infer that $\tip(Penguin) \sqsubseteq \bot $, and we admit that there can be typical penguins.
Furthermore, we conclude that they are $C$s ($\tip(Penguin) \sqsubseteq C $). 

Clearly, this comes at the price of considering all the maximal sets of defeasible inclusions
compatible with $\mathit{Penguin}$ (as in the lexicographic closure).  
A weaker and more skeptical variant of the MP-closure 
could be defined along the lines of \cite{GiordanoICTCS2017},
where an alternative notion of closure, the {\em skeptical closure}, is proposed 
which is weaker than the lexicographic closure and its computation does not require to generate all the alternative maximally consistent bases.
The construction in \cite{GiordanoICTCS2017} is based on the idea of building a single base, i.e. a single maximal consistent set of defeasible inclusions, 
starting with the defeasible inclusions with highest rank and progressively adding less specific inclusions, when consistent,
but excluding the defeasible inclusions which produce a conflict at a certain stage
without considering all the alternative consistent bases. 

\normalcolor

\vspace{-0.2cm}

\section{Related Work} \label{Related work}

A lot of work  has  been done in order to extend the basic formalism of Description Logics (DLs)  with nonmonotonic reasoning features \cite{Straccia93,baader95b,donini2002,eiter2004,lpar2007,AIJ13,kesattler,sudafricaniKR,bonattilutz,casinistraccia2010,rosatiacm,hitzlerdl,KnorrECAI12,CasiniDL2013,bonattiAIJ15}. 
The purpose of these extensions is to allow reasoning about \emph{prototypical properties} of individuals or classes of individuals. 
A detailed descriptions of these formalisms and of their relations to our approach based on the $\tip$ operator, on preferential semantics and on minimal models, may be found in \cite{AIJ13,AIJ15}.
Further recent approaches to defeasible inference deal with low complexity description logics \cite{Bonatti2011,ijcai2011,Bozzato14,Papini2016,GiordanoDL2016,Pensel2017}.

The interest of rational closure for DLs is that it 
provides a significant and reasonable skeptical nonmonotonic inference mechanism, 
while keeping the same complexity as the underlying logic.
The first notion of rational closure for DLs was defined by Casini and Straccia \cite{casinistraccia2010}. 
Their rational closure construction for $\alc$  directly uses entailment in $\alc$ over a materialization of the KB.
A variant of this notion of rational closure has been studied in \cite{CasiniDL2013}, and a semantic characterization for it has been proposed.
In \cite{dl2013,AIJ15} a notion of rational closure for the logic $\alc$ has been proposed, building on the notion of rational closure proposed by Lehmann and Magidor \cite{whatdoes}, together with a minimal model semantics characterization.

To overcome the limitations of rational closure,  and, in particular, the fact that one cannot separately reason property by property,
the lexicographic closure has been introduced for the description logic $\alc$ by Casini and Straccia in \cite{Casinistraccia2012},
as a generalization of the lexicographic construction by Lehmann \cite{Lehmann95}.
A detailed comparison between the present proposal and lexicographic closure has been presented in Section \ref{sec:Lexicographic}.
To cope with the limitations of rational closure, in \cite{Casinistraccia2011,CasiniJAIR2013} an approach based on the combination of rational closure and \emph{Defeasible Inheritance Networks} has also been developed, and in \cite{Casini2014} a notion of relevant closure has been introduced,
defining defeasibility in terms of justifications.

In \cite{bonattiAIJ15} a non monotonic description logics ${\cal DL}^N$ has been proposed, which supports normality concepts and enjoys good computational properties. In particular, ${\cal DL}^N$ preserves the tractability of low complexity DLs, including ${\el}^{++}$ and $DL$-$lite$.
The logic incorporates a notion of overriding,  namely the idea that more specific inclusions override less specific ones.
A  difference with rational closure is that, in case there are unresolved conflicts among defeasible inclusions with the same preference, in ${\cal DL}^N$ inheritance is not blocked and the conflict is made explicit through the inconsistency of some normality concept. 
As we have seen, this logic also allows to separately reason with different aspects and their inheritance 
through the overriding of conflicting properties. A detailed comparison can be found in Section \ref{sec:Lexicographic}.


An approach related to our approach is given in \cite{fernandez-gil}, where an extension of $\alct$ with several typicality operators, each corresponding to a preference relation is proposed.
This approach is related to ours although  different: the language in  \cite{fernandez-gil}  allows for several typicality operators whereas  we only have a single typicality operator.  The focus of  \cite{fernandez-gil}  is indeed different from ours, as it does not deal with rational closure, whereas this is one of the main contributions of our paper.

Britz and Varzinczak in \cite{BritzDL2017} deal with a notion of normality for roles by parameterizing 
preference order on binary relations in the domain of interpretation.
They allow the use of defeasible roles in complex concepts, as well as in defeasible (concept and role) subsumptions, and
in defeasible role assertions. The work in  \cite{BritzDL2017} extends the notion of defeasible KB in another direction with respect to our proposal,
by allowing multiple preference relations (parametrized by roles) among pairs of individuals (rather than among single individuals). 


To conclude this section, let us observe that the S-enriched semantics introduced in this paper is stronger than the Enriched semantics in \cite{GliozziAIIA2016}, which allows for a 
weaker condition than (4). 
Indeed, minimal entailment in \cite{GliozziAIIA2016} is stronger than entailment under the rational closure, but weaker than minimal entailment under the S-enriched semantics. 
In particular, in Example \ref{exa:BabyPenguin}, minimal entailment in \cite{GliozziAIIA2016} does not allow the conclusion that typical baby penguins do not fly. 
\normalcolor

\section{Conclusions} \label{Conclusions}

In this paper, we have introduced a multipreference semantics for the description logic $\alc$ which is a refinement of the semantics for rational closure 
proposed in \cite{AIJ15}, Beside a global preference relation,  models are equipped with several preference relations indexed with aspects.
In this way, the multipreference semantics allows to reason about the inheritance of different properties in an ontology in a separate way. 
We have introduced a closure construction which is sound with respect to the multipreference semantics,
and we have established comparisons with the lexicographic closure and with the logic ${\cal DL}^N$.

Verifying whether a query $\tip(B) \sqsubseteq D$ 
is derivable from the MP-closure of the TBox
in the worst case requires to consider an exponential number of sets of defeasible inclusions.
However, we expect that a weaker closure construction, 
which only requires a polynomial number of entailment checks in the underlying description logic,
can be defined in the style of the skeptical construction in \cite{GiordanoICTCS2017} mentioned above.
This will be the subject of future work.

Another subject for future work will be the investigation of a possible extension of this semantics to more expressive description logics,
for those cases in which the rational closure can be consistently defined, as established in \cite{CILC2015SI}.

\section*{Acknowledgements}
\noindent
This research is partially supported by INDAM-GNCS Project 2016 ``Ragionamento Defeasible nelle Logiche Descrittive".

\bibliographystyle{plain}

\begin{thebibliography}{10}

\bibitem{baader95b}
F.~Baader and B.~Hollunder.
\newblock Priorities on defaults with prerequisites, and their application in
  treating specificity in terminological default logic.
\newblock {\em Journal of Automated Reasoning (JAR)}, 15(1):41--68, 1995.

\bibitem{Papini2016}
J.F. Baget, S.~Benferhat, Z.~Bouraoui, M.~Croitoru, M.L. Mugnier, O.~Papini,
  S.~Rocher, and K.~Tabia.
\newblock Inconsistency-tolerant query answering: Rationality properties and
  computational complexity analysis.
\newblock In {\em Logics in Artificial Intelligence - 15th European Conference,
  {JELIA} 2016, Larnaca, Cyprus, November 9-11, 2016, Proceedings}, pages
  64--80, 2016.

\bibitem{bonattiAIJ15}
P.~A. Bonatti, M.~Faella, I.~Petrova, and L.~Sauro.
\newblock A new semantics for overriding in description logics.
\newblock {\em Artif. Intell.}, 222:1--48, 2015.

\bibitem{Bonatti2011}
P.~A. Bonatti, M.~Faella, and L.~Sauro.
\newblock Defeasible inclusions in low-complexity dls.
\newblock {\em J. Artif. Intell. Res. {(JAIR)}}, 42:719--764, 2011.

\bibitem{bonattilutz}
Piero~A. Bonatti, Carsten Lutz, and Frank Wolter.
\newblock {The Complexity of Circumscription in DLs}.
\newblock {\em Journal of Artificial Intelligence Research (JAIR)},
  35:717--773, 2009.

\bibitem{Bozzato14}
L.~Bozzato, T.~Eiter, and L.~Serafini.
\newblock Contextualized knowledge repositories with justifiable exceptions.
\newblock In {\em DL 2014}, volume 1193 of {\em {CEUR} Workshop Proceedings},
  pages 112--123, 2014.

\bibitem{sudafricaniKR}
Katarina Britz, Johannes Heidema, and Thomas Meyer.
\newblock Semantic preferential subsumption.
\newblock In G.~Brewka and J.~Lang, editors, {\em Principles of Knowledge
  Representation and Reasoning: Proceedings of the 11th International
  Conference (KR 2008)}, pages 476--484, Sidney, Australia, September 2008.
  AAAI Press.

\bibitem{BritzDL2017}
Katarina Britz and Ivan~Jos{\'{e}} Varzinczak.
\newblock Towards defeasible {SROIQ}.
\newblock In {\em Proceedings of the 30th International Workshop on Description
  Logics, Montpellier, France, July 18-21, 2017.}, 2017.

\bibitem{Casini2014}
G.~Casini, T.~Meyer, K.~Moodley, and R.~Nortje.
\newblock Relevant closure: {A} new form of defeasible reasoning for
  description logics.
\newblock In {\em {JELIA} 2014}, LNCS 8761, pages 92--106. Springer, 2014.

\bibitem{CasiniDL2013}
G.~Casini, T.~Meyer, I.~J. Varzinczak, , and K.~Moodley.
\newblock {Nonmonotonic Reasoning in Description Logics: Rational Closure for
  the ABox}.
\newblock In {\em DL 2013, 26th International Workshop on Description Logics},
  volume 1014 of {\em CEUR Workshop Proceedings}, pages 600--615. CEUR-WS.org,
  2013.

\bibitem{casinistraccia2010}
G.~Casini and U.~Straccia.
\newblock {Rational Closure for Defeasible Description Logics}.
\newblock In T.~Janhunen and I.~Niemel{\"a}, editors, {\em Proceedings of the
  12th European Conference on Logics in Artificial Intelligence (JELIA 2010)},
  volume 6341 of {\em Lecture Notes in Artificial Intelligence}, pages 77--90,
  Helsinki, Finland, September 2010. Springer.

\bibitem{Casinistraccia2012}
G.~Casini and U.~Straccia.
\newblock {Lexicographic Closure for Defeasible Description Logics}.
\newblock In {\em Proc. of Australasian Ontology Workshop, vol.969}, pages
  28--39, 2012.

\bibitem{CasiniJAIR2013}
G.~Casini and U.~Straccia.
\newblock {Defeasible inheritance-based description logics}.
\newblock {\em Journal of Artificial Intelligence Research (JAIR)},
  48:415--473, 2013.

\bibitem{Casinistraccia2011}
Giovanni Casini and Umberto Straccia.
\newblock {Defeasible Inheritance-Based Description Logics}.
\newblock In Toby Walsh, editor, {\em Proceedings of the 22nd International
  Joint Conference on Artificial Intelligence (IJCAI 2011)}, pages 813--818,
  Barcelona, Spain, July 2011. Morgan Kaufmann.

\bibitem{donini2002}
F.~M. Donini, D.~Nardi, and R.~Rosati.
\newblock {Description logics of minimal knowledge and negation as failure}.
\newblock {\em {ACM Transactions on Computational Logic (ToCL)}},
  3(2):177--225, 2002.

\bibitem{eiter2004}
T.~Eiter, T.~Lukasiewicz, R.~Schindlauer, and H.~Tompits.
\newblock {Combining Answer Set Programming with Description Logics for the
  Semantic Web}.
\newblock In D.~Dubois, C.A. Welty, and M.~Williams, editors, {\em Principles
  of Knowledge Representation and Reasoning: Proceedings of the 9th
  International Conference (KR 2004)}, pages 141--151, Whistler, Canada, June
  2004. AAAI Press.

\bibitem{fernandez-gil}
Oliver~Fernandez Gil.
\newblock On the non-monotonic description logic alc+t\({}_{\mbox{min}}\).
\newblock {\em CoRR}, abs/1404.6566, 2014.

\bibitem{ISMIS2015}
L.~Giordano and V.~Gliozzi.
\newblock Encoding a preferential extension of the description logic
  \emph{SROIQ} into \emph{SROIQ}.
\newblock In {\em Proc. {ISMIS} 2015}, volume 9384 of {\em LNCS}, pages
  248--258. Springer, 2015.

\bibitem{CILC2015SI}
L.~Giordano, V.~Gliozzi, and N.~Olivetti.
\newblock Towards a rational closure for expressive description logics: the
  case of $\shiq$.
\newblock In {\em to appear in Fundamenta Informaticae}, 2017.

\bibitem{lpar2007}
L.~Giordano, V.~Gliozzi, N.~Olivetti, and G.~L. Pozzato.
\newblock {P}referential {D}escription {L}ogics.
\newblock In Nachum Dershowitz and Andrei Voronkov, editors, {\em Proceedings
  of LPAR 2007 (14th Conference on Logic for Programming, Artificial
  Intelligence, and Reasoning)}, volume 4790 of {\em LNAI}, pages 257--272,
  Yerevan, Armenia, October 2007. Springer-Verlag.

\bibitem{FI09}
L.~Giordano, V.~Gliozzi, N.~Olivetti, and G.~L. Pozzato.
\newblock {ALC+T}: a preferential extension of {D}escription {L}ogics.
\newblock {\em Fundamenta Informaticae}, 96:1--32, 2009.

\bibitem{ijcai2011}
L.~Giordano, V.~Gliozzi, N.~Olivetti, and G.~L. Pozzato.
\newblock {Reasoning about typicality in low complexity DLs: the logics $\eltm$
  and $\dlltm$}.
\newblock In Toby Walsh, editor, {\em Proceedings of the 22nd International
  Joint Conference on Artificial Intelligence (IJCAI 2011)}, pages 894--899,
  Barcelona, Spain, July 2011. Morgan Kaufmann.

\bibitem{AIJ13}
L.~Giordano, V.~Gliozzi, N.~Olivetti, and G.~L. Pozzato.
\newblock {A NonMonotonic Description Logic for Reasoning About Typicality}.
\newblock {\em Artificial Intelligence}, 195:165--202, 2013.

\bibitem{AIJ15}
L.~Giordano, V.~Gliozzi, N.~Olivetti, and G.~L. Pozzato.
\newblock {Semantic characterization of rational closure: From propositional
  logic to description logics}.
\newblock {\em Artificial Intelligence}, 226:1--33, 2015.

\bibitem{GiordanoDL2016}
L.~Giordano and D.~Theseider~Dupr\'e.
\newblock {Reasoning in a Rational Extension of SROEL}.
\newblock In {\em DL2016}, volume 1577 of {\em {CEUR} Workshop Proceedings},
  2016.

\bibitem{GiordanoICTCS2017}
Laura Giordano.
\newblock Reasoning about exceptions in ontologies: a skeptical preferential
  approach (extended abstract).
\newblock In {\em Joint Proceedings of the 18th Italian Conference on
  Theoretical Computer Science and the 32nd Italian Conference on Computational
  Logic, Naples, Italy, September 26-28, 2017}, volume 1949 of {\em CEUR
  Workshop Proceedings}, pages 6--10, 2017.

\bibitem{dl2013}
Laura Giordano, Valentina Gliozzi, Nicola Olivetti, and Gian~Luca Pozzato.
\newblock { Minimal Model Semantics and Rational Closure in Description Logics
  }.
\newblock In {\em 26th International Workshop on Description Logics (DL 2013)},
  volume 1014, pages 168 -- 180, 7 2013.

\bibitem{GliozziAIIA2016}
Valentina Gliozzi.
\newblock Reasoning about multiple aspects in rational closure for dls.
\newblock In {\em AI*IA 2016: Advances in Artificial Intelligence - XVth
  International Conference of the Italian Association for Artificial
  Intelligence, Genova, Italy, November 29 - December 1, 2016, Proceedings},
  pages 392--405, 2016.

\bibitem{kesattler}
P.~Ke and U.~Sattler.
\newblock {N}ext {S}teps for {D}escription {L}ogics of {M}inimal {K}nowledge
  and {N}egation as {F}ailure.
\newblock In F.~Baader, C.~Lutz, and B.~Motik, editors, {\em Proceedings of
  Description Logics}, volume 353 of {\em CEUR Workshop Proceedings}, Dresden,
  Germany, May 2008. CEUR-WS.org.

\bibitem{KnorrECAI12}
M.~Knorr, P.~Hitzler, and F.~Maier.
\newblock Reconciling owl and non-monotonic rules for the semantic web.
\newblock In {\em ECAI 2012}, page 474Ð479, 2012.

\bibitem{hitzlerdl}
Adila~Alfa Krisnadhi, Kunal Sengupta, and Pascal Hitzler.
\newblock Local closed world semantics: Keep it simple, stupid!
\newblock In {\em Proceedings of Description Logics}, volume 745 of {\em CEUR
  Workshop Proceedings}, Barcelona, Spain, July 2011.

\bibitem{whatdoes}
Daniel Lehmann and Menachem Magidor.
\newblock What does a conditional knowledge base entail?
\newblock {\em Artificial Intelligence}, 55(1):1--60, 1992.

\bibitem{Lehmann95}
Daniel~J. Lehmann.
\newblock Another perspective on default reasoning.
\newblock {\em Ann. Math. Artif. Intell.}, 15(1):61--82, 1995.

\bibitem{rosatiacm}
Boris Motik and Riccardo Rosati.
\newblock {Reconciling Description Logics and rules}.
\newblock {\em Journal of the ACM}, 57(5), 2010.

\bibitem{Pensel2017}
Maximilian Pensel and Anni{-}Yasmin Turhan.
\newblock Including quantification in defeasible reasoning for the description
  logic el\({}_{\mbox{{\(\perp\)}}}\).
\newblock In {\em Logic Programming and Nonmonotonic Reasoning - 14th
  International Conference, {LPNMR} 2017, Espoo, Finland, July 3-6, 2017,
  Proceedings}, pages 78--84, 2017.

\bibitem{Straccia93}
U.~Straccia.
\newblock Default inheritance reasoning in hybrid kl-one-style logics.
\newblock In R.~Bajcsy, editor, {\em Proceedings of the 13th International
  Joint Conference on Artificial Intelligence (IJCAI 1993)}, pages 676--681,
  Chamb\'ery, France, August 1993. Morgan Kaufmann.

\end{thebibliography}

\end{document}